\documentclass[letterpaper, 10 pt, conference]{ieeeconf}  

\IEEEoverridecommandlockouts                              

\usepackage{amsmath,amsfonts}
\usepackage[algo2e,ruled,vlined,linesnumbered,resetcount,norelsize]{algorithm2e}

\usepackage{array}
\usepackage[caption=false,font=normalsize,labelfont=sf,textfont=sf]{subfig}
\usepackage{textcomp}
\usepackage{stfloats}
\usepackage{float}
\usepackage{url}
\usepackage{bm}
\usepackage{verbatim}
\usepackage{graphicx}
\usepackage{cite}
\usepackage{mathtools} 
\usepackage{amssymb}  

\usepackage{amsthm}
\usepackage{algpseudocode}
\usepackage{xcolor}
\usepackage{accents}
\usepackage{mathrsfs}
\usepackage[normalem]{ulem}
\usepackage{soul}
\usepackage{booktabs}

\usepackage{multirow}

\usepackage[T1]{fontenc}
\usepackage[utf8]{inputenc}
\usepackage[skip=0.333\baselineskip]{caption}
\usepackage{siunitx} 
\newcolumntype{T}{S[table-format=3.3, input-symbols={()},
                    table-space-text-post={$^{***}$},
                    table-align-text-post=false]}
\usepackage[colorlinks=true, allcolors=black]{hyperref}
\usepackage{tabularx,booktabs}
\newcolumntype{C}{>{\centering\arraybackslash}X} 
\renewcommand{\baselinestretch}{.98}

\newtheorem{proposition}{Proposition }

\newtheorem{remark}{Remark}

\definecolor{green}{RGB}{11,155,13}

\newcommand{\marginXW}[1]{\marginpar{\color{purple}\tiny\ttfamily Xuan: #1}}

\SetCommentSty{mycommfont}
\newcommand{\longthmtitle}[1]{\mbox{} \emph{(#1):}}
\SetKwRepeat{Do}{do}{while}

\title{\LARGE \bf
Learning Coordinated Maneuver in Adversarial Environments
}

\author{Zechen Hu, Manshi Limbu, Daigo Shishika, Xuesu Xiao, and Xuan Wang
\thanks{Work supported by  Army Research Office (W911NF-22-2-0242) and NSF (2332210). George Mason University. {\tt\scriptsize \{zhu3, klimbu2, dshishik, xiao, xwang64\}@gmu.edu}. 
}
}

\begin{document}
\maketitle
\thispagestyle{empty}
\pagestyle{empty}

 \begin{abstract}
This paper aims to solve the coordination of a team of robots traversing a route in the presence of adversaries with random positions. Our goal is to minimize the overall cost of the team, which is determined by (i) the accumulated risk when robots stay in adversary-impacted zones and (ii) the mission completion time. During traversal, robots can reduce their speed and act as a `guard' (the slower, the better), which will decrease the risks certain adversary incurs. 
This leads to a trade-off between the robots' guarding behaviors and their travel speeds. The formulated problem is highly non-convex and cannot be efficiently solved by existing algorithms. Our approach includes a theoretical analysis of the robots' behaviors for the single-adversary case. As the scale of the problem expands, solving the optimal solution using optimization approaches is challenging, therefore, we employ reinforcement learning techniques by developing new encoding and policy-generating methods. Simulations demonstrate that our learning methods can efficiently produce team coordination behaviors. We discuss the reasoning behind these behaviors and explain why they reduce the overall team cost.    
\end{abstract}

\section{Introduction}


Coordination of multi-robot systems has been studied under various contexts~\cite{DA-KS-JR:18}, including cooperative path planning~\cite{TA-OE-KA-SM:17,wang2022d3g}, resource sharing and task allocation~\cite{AM-JJ-RA-RA-WJ-HE:21}, and geometric formation maintenance~\cite{OK-PMC-AH:15,zhou2023distributed}. Complementary to the challenges addressed in these works, this paper introduces a new problem centered around generating coordinated team behaviors to reduce risks caused by adversaries.
Considering a graph-based representation, team coordination problems have been studied in~\cite{DCA-WKC-MJ:23,LM-HZ-OS:23,dimmig2023uncertainty}. In this work, as shown in Fig.~\ref{Fig_intro}, we consider a route-based version of the problem, fine-grinding the movements of robots in continuous space. The environment features adversaries whose positions are randomly initialized from a set. Robots accumulate `risks' when traveling through adversary-controlled zones, and such risks can be reduced by robots if they slow down and `guard' a certain adversary. We define the team cost as a combination of mission completion time and accumulated risk. Therefore, minimizing this cost requires robots' coordination, trading off between their speed and the adoption of guarding behaviors. This adds a novel dimension of complexity and strategic decision-making, which is unattended in conventional multi-robot coordination tasks.

The described scenario has direct applications in many domains. For instance, when multiple rescue robots need to pass a fire-engulfed corridor~\cite{bogue2021role}, some robots might deploy countermeasures as \textit{guard} to quell flames, ensuring safer passage for their peers; On battlefields, when multiple vehicles need to traverse enemy-controlled territories~\cite{matyas2019brief}, suppressive fire from allies can guard and mitigate threats posed by the enemy. The formulated multi-robot coordination problem is challenging due to the combinatorial nature of robots' states, their hybrid actions (speed and guard), and multiple constraints embedded through robot-adversary interactions.

\begin{figure}[t]
    \centering
	\includegraphics[width=0.45\textwidth]{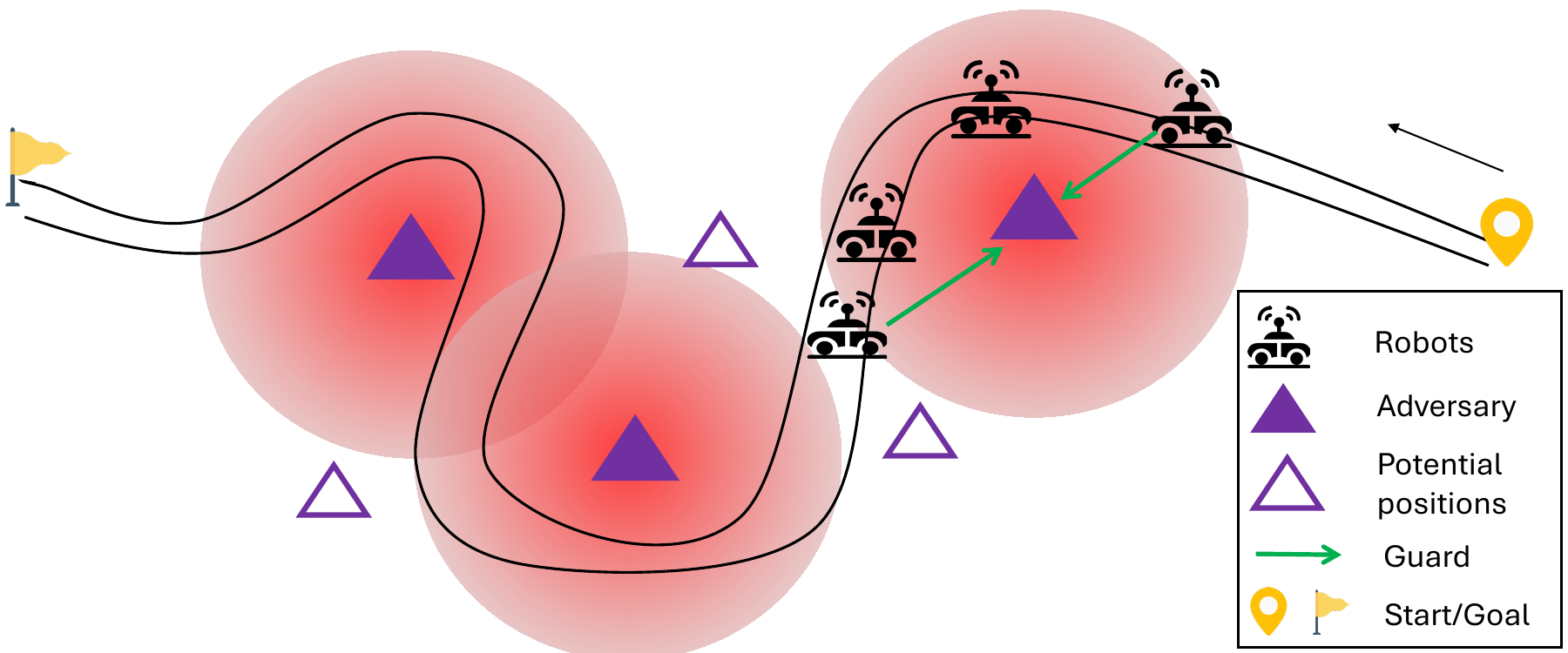}
	\caption{A team of robots traversing an environment with risk. Adversary positions are randomly initialized from a set.}
 \vspace{-4ex}
	\label{Fig_intro}
\end{figure}

To solve this problem, one feasible approach from the optimization literature is Mixed Integer Programming (MIP). However, since MIP for multi-robot coordination generally provides an \textit{open-loop} solution for the full trajectories of all robots~\cite{cauligi2020learning}, it faces scalability challenges as the number of robots increases, and requires repeated re-planning if the adversary position or any environmental features change.
Motivated by these challenges, in this paper, we seek \textit{closed-loop} policies using a Reinforcement Learning (RL)-based approach. The \textit{main contributions} are as follows: (i) We rigorously formulate a new multi-robot coordination problem that incorporates guard behaviors among team members to mitigate risks from adversaries.
(ii) We investigate the conversion of the problem into Markov Decision Processes (MDPs) with hybrid move and guard actions. 
(iii) We introduce a Hybrid Proximal Policy Optimization algorithm tailored to our problem, featuring special treatment of reward reshaping and a unique multi-weighted hot encoding for representing robots' states and adversary positions.
(iv) We perform extensive simulated experiments to validate the effectiveness of the proposed method and compare it with MIP methods. We also elucidate the rationale behind observed behaviors and their effectiveness in reducing the collective team cost.



\section{Literature review}
We review related work on MIP and RL for solving multi-robot coordination problems.

\subsection{Mixed Integer Programming}
MIP has been applied to various multi-robot coordination problems, including task allocation~\cite{ZC-SJA:16,wang2021distributed}, multi-robot path planning~\cite{dimmig2023uncertainty}, and environmental coverage and exploration~\cite{GF-MA-MI-MS-IT:22}.  
Solving MIP problems is NP-complete, making traditional MIP solvers sensitive to the number of variables~\cite{ioan2021mixed} and primarily applicable to small-scale problems. However, for our problem of interest, the coordination of robots occurs at every time step, depending on their coupled actions and states. Such constraints and task specifications can be encoded through logic formulation and piece-wise nonlinear functions~\cite{calegari2021logic}.
Yet, considering the entire trajectories of all robots as variables can quickly render MINLP (Mixed Integral NonLinear Programming) with coupled constraints intractable. To address this issue, the advances in MIP such as interior point~\cite{PFA-WSJ:20}, branch and bound~\cite{BLD-HDC-MRA-HJD:18,DCA-WKC-MJ:23} methods, and heuristic approaches~\cite{yu2016optimal}, can drastically improve computation scalability by leveraging convexity properties or the decomposability of the problem. While these properties may not inherently exist for general applications, one can introduce special cost function design and relaxation~\cite{DCA-WKC-MJ:23} to promote convexity; or by simplifying explicit collaboration between robots~\cite{yu2016optimal} to enhance problem decomposability.

In this paper, these techniques~\cite{DCA-WKC-MJ:23,yu2016optimal} are not directly applicable since explicit coordination is critical to formulating robots' coordination and impact significantly on the optimal behavior of the team\footnote{In simulated experiments, a naive baseline algorithm with simplified robot coordination will be given to reflect this gap}. On the other hand, given the hybrid action spaces of the robots, even simple linear risk functions can lead to highly non-smooth non-convex objectives and constraints, which can make MINLP solvers numerically unstable, or converge to sub-optimal solutions~\cite{KD:14}.
Finally, our problem requires fast deployment and quick adaptation to adversaries with random positions, which favors closed-loop solutions over the open-loop solutions provided by MINLP.

\subsection{Multi-Agent Reinforcement Learning}
While optimization techniques suffer from computational complexity, Reinforcement Learning (RL) methods enable robots to employ trial and error to efficiently find empirical solutions for complex problems. Moreover, the learned policy is closed-loop and can adapt in real-time to new adversary positions.
For centralized problem solving, Deep Q-Networks (DQN)\cite{OI-BC-PA-VRB:16}, a value-based RL method suitable for discrete actions, has been applied to learn team formations in battle games~\cite{DE-ST:18}. For complex tasks with continuous actions such as traffic optimization~\cite{ZH-CW-HZ-LM-YY:20}, Advantage Actor-Critic (A2C) methods\cite{KV-TJ:99} can achieve faster convergence and better exploration by utilizing a policy-based model as an actor. Building on A2C, there are generalizations such as Proximal Policy Optimization (PPO)\cite{SJ-WF-DP-RA-KO:17} and Deep Deterministic Policy Gradient (DDPG)~\cite{LS-WY-CX-DH-FF-RS:19} with improved stability or data efficiency.
For systems with hybrid action space, there exist hybrid-PPO\cite{FZ-SR-ZW-YY:19} methods that can output discrete actions simultaneously with continuous actions.
In addition, considering the agent-based nature of our problem, the mentioned algorithms also have multi-agent variants such as Deep Coordination Graph~\cite{BW-KV-WS:20}, Multi-agent PPO~\cite{WJ-SL:20},  and Multi-agent DDPG~\cite{LR-WY-TA-HI-PA:17}, allowing for decentralized execution, where each robot learns a local model and determines actions according to local observation. The key advantage of MARL is to improve algorithmic scalability. However, the solution may be sub-optimal due to partial information. 

In our setup, although we follow an idea similar to existing centralized hybrid-PPO methods, we face challenges in training efficiency. This has motivated us to apply special treatment to the state encoding of robots and adversaries, as well as reward reshaping, to address these issues.

 \section{Problem Formulation}\label{sec_pf}
In this section, we will first formulate the \textit{task} of the multi-robot team, then introduce the notions of \textit{risk} and \textit{guard}. Based on these, we quantify the \textit{team cost} and describe our \textit{problem of interest}. Throughout the following definitions, adversaries are considered to be heterogeneous while robots are homogeneous. 
 
\noindent\textbf{Robots:} Consider a number of $n$ robots traversing a route of length $L$.
The position of the $i$-th robot along the route at time $t$ is represented by $s_i^t\in [0, L]$. Let $v_i^t\in[0,v_{\max}]$ denote the speed of the $i$-th traveling robot at time $t$. Thus, the position update for each traveling robot is given as
 \begin{align}\label{eq_dyn}
     s_i^{t+1} = s_i^t + v_i^t \Delta t.
 \end{align}
where $\Delta t$ is the time interval. 
Before robot $i$ arrives at the destination, each time step will produce a time penalty, denoted by $P_i^t$. We define
\begin{align}\label{eq_def_pi}
P_i^t = \begin{cases}p & \text{if}~ s_i^t < L,\\
    0& \text{if} ~s_i^t = L.
    \end{cases}
\end{align}

\noindent\textbf{Adversaries and Risk:} Let $m$ denote the number of adversaries. The position of adversary $j$ is represented as 
\begin{align}\label{eq_adv}
z_j\in \mathcal{D}_j\subset(0, L),
\end{align}
which is randomly chosen from a set $\mathcal{D}_j$ for each trial of experiment.
Around $z_j$ is the impact zone of the adversary, denoted by $\mathcal{M}_j$. If a robot is in this region, i.e., $s_i^t\in \mathcal{M}_j$, a cost $r_{i,j}^t$ will be incurred, 
 \begin{align}\label{eq_def_r}
     r_{i,j}^t = 
     \begin{cases}
        f_j(s_i^t,z_j)\ge0& \text{if}~s_i^t\in \mathcal{M}_j, \\
        0 & \text{if}~s_i^t\notin \mathcal{M}_j,
     \end{cases}
 \end{align}
which depends on the relative positions of $s_i^t$ and $z_j^t$.

\noindent\textbf{Guard:} During traversal, if $s_k^t\in\mathcal{M}_j$, robot $k\in\{1,\cdots,n\}$ can reduce its speed and counteract adversary $j$ as a `guard'. 
Specifically, let  $g^t_k\in\{1, 2,\cdots,m\}$ denote the index of the adversary that robot $k$ is guarding against at time $t$.
Then, the risks that adversary $j$ incurs to robots $i,~\forall s_i^t \in \mathcal{M}_j$ are discounted to $\alpha_{k,j}^t r_{i,j}^t$, where
\begin{align}\label{eq_discount}
    \alpha_{k,j}^t = \begin{cases}
        1- \beta \displaystyle \frac{|v_{\max}-v_k^t|}{v_{\max}}&\text{if}~g_k^t=j,\\
        1 &\text{otherwise},
    \end{cases}
\end{align}
is the discount factor with $\beta\in(0,1)$. When $v_k^t=0$, robot $k$ achieves best guarding performance $\alpha^t_{k,j}=1-\beta$, while as $v_k^t\to v_{\max}$, the guarding effect vanishes. Furthermore, we assume the guarding effects stack with each other, thus, considering all robots in the system guarding an adversary $j$, the risk it incurs to robot $i$ is discounted by all guards as
    $\prod_{k=1}^n \alpha_{k,j}^t r_{i,j}^t$.

\noindent{\textbf{Team Cost:}} 
Let $T$ be the total time for all robots in the team to traverse the route.
Based on the above definitions, the team cost considers the risks and time penalties accumulated by all robots,
\begin{align}\label{eq_def_f}
    \bm{J} = \sum_{i=1}^n (\bm{R}_i + \bm{P}_i),
\end{align}
where $\bm{R}_i = \sum_{t=1}^T R_i^t \Delta t$ and $\bm{P}_i = \sum_{t=1}^T P_i^t \Delta t$, with
\begin{align}\label{eq_def_Ri}
    \quad R_i^t=\sum_{j=1}^m \prod_{k=1}^n \alpha_{k,j}^t r_{i,j}^t.
\end{align}
taking into account the guarding effect.

\noindent{\textbf{Problem of Interest:}} In each time step, robot $i$'s action is composed of traveling speed $v_{i}^t$ and guard target $g_i^t$. Assume the robots can observe adversaries' positions $z_j$ but do not know $\mathcal{D}_j$. To strategically design all robots' behaviors,
let $\bm{v}^t = \{v_1^t, \dots, v_n^{t}\}$, and $\bm{g}^t = \{g_1^t, \dots, g_n^{t}\}$.
The problem is to minimize the team cost, i.e., $t\in\{1,\cdots,T\}$
\begin{align}
    \min_{\{\bm{v}^t,\bm{g}^t\}} \bm{J}.
\end{align}

Note that the moving and guarding behaviors of robots lead to team coordination, as one robot can decrease its speed to benefit all traveling robots within the influence region of the guarded adversary.

 \section{method}\label{Sec_method}
In this section, we present methods for solving the formulated problem. We perform a simple analysis for the case of single-robot single-adversary. 
For more complicated cases, we introduce proximal policy optimization (PPO) based RL algorithms with a special multi-weighted hot state encoding mechanism and reward reshaping to improve training efficiency.

\subsection{An observation for the single adversary case}\label{Sec_theory}
We start by considering a single robot and a single adversary case that makes the robot's behavior analyzable. Although this analysis does not yield directly applicable robot coordination strategies, its conclusions align with several real-world observations and can offer insights into the multiple robots and adversaries cases.

Given that the robot's velocity affects its guarding performance $\alpha_{i,j}^t$ as shown in \eqref{eq_discount}, the single robot faces a trade-off: slowing down to incur a discount factor $\alpha_{i,j}^t$ to the risk it takes or speeding up to decrease the time it stays in the adversary controlled area.

\begin{proposition}\label{lm_1}
Suppose there's only one robot and one adversity. Consider the guard discount defined in \eqref{eq_discount}.
If $\Delta t \to 0$, the cost $\bm{J}$ is minimized if the robot always maintains maximum speed, i.e.,  $\forall t$, $v=v_{\max}$.
\end{proposition}

\begin{proof}
For ease of performing analysis, given $\Delta t \to 0$, we convert the risk accumulation to a continuous form, which reads\footnote{The subscripts $i,j$ are omitted since only considering one robot and one adversary. For continuous representation, the notation for time $t$ is also omitted.}
\begin{align}\label{eq_cont}
    \bm{R}&=\int_{0}^{T}(1-\beta+\frac{\beta v}{v_{\max}})f(s,z)dt\nonumber\\
    &=\int_{0}^{L}(1-\beta+\beta\frac{v}{v_{\max}})f(s,z)\frac{ds}{v}\nonumber\\
    &= \int_{0}^{L}(\frac{1-\beta}{v}+\frac{\beta}{v_{\max}})f(s,z){ds}
\end{align}
where $v>0$ and $\int_0^{T}s ~dt=L$. In the second line of the equation, we substitute the integration variable using the property $\frac{d{s}}{dt}=v$. From \eqref{eq_cont}, it's clear that $\bm{R}$ is minimized if $\forall t, v=v_{\max}$. 
Furthermore, given the definition of $\bm{P}$ as the time penalty, it is also minimized by maintaining $v=v_{\max}$. This completes the proof.
\end{proof}

Proposition~\ref{lm_1} only analyzes a single-robot single-adversary case. 
However, in the multi-robot case, if there exist spots $s_i^t \in \mathcal{M}_j$ such that $f_j(s_i^t, z_j^t) = 0$, indicating that some robots can guard moving robots with $0$ risk, then they should stop to guard, and the moving robots should adopt the strategy of Proposition \ref{lm_1} to reduce $\sum_{i=1}^n \bm{R}_i$. This leads to a `bang-bang behavior'~\cite{BR-GI-GO:56}, where the robot either remains in a `safe spot' to guard others or moves at full speed when under the protection of other robots. This moving pattern, known as `bounding overwatch', is well-justified in the military domain. However, this strategy also incurs significant time costs in terms of $\sum_{i=1}^n \bm{P}_i$, as it necessitates some robots to fully stop when guarding others.

When the number of robots grows large and the time penalty $\sum_{i=1}^n \bm{P}_i$ begins to dominate the overall cost, the strategy might change into two possible variations: (i) the robots take special scheduling to stop or move, and when moving, they move at full speed; (ii) several robots move at intermediate speeds to perform move and guard simultaneously. Both strategies are observed later in our simulations section, depending on the environment setup. Nevertheless, these behaviors consider the movements of multiple robots in a coupled manner, which makes theoretical analysis intractable. In addition, as will be discussed in the simulation, due to the piece-wise and condition-dependent non-linear cost structure and coupled constraints, the MINLP formulation of the problem is very difficult to solve.
Motivated by these, we seek to use reinforcement learning to solve the problem with a closed-loop solution. 

\subsection{MDP Formulation and RL methods}\label{sec_RL}
The Markov Decision Process (MDP) formulation of our problem is defined by the tuple $(\mathcal{S}, \mathcal{A}, \mathcal{T}, \gamma, {R})$, including state, action, state transition, discount factor, and reward. Here, we propose a centralized MDP to address the multi-robot coordination problem formulated in Sec.~\ref{sec_pf}. 
Specifically, let $\bm{s}^t = \{s_1^t, \dots, s_n^t, z_1, \dots, z_m\} \in \mathcal{S}$ be the state set at time $t$, where $s_i^t,z_j\in[0,L]$ are associated with robots positions and adversaries positions, respectively. 
Since adversaries may appear at random positions, their information needs to be encoded into the system state so that the learned policy can generate different strategies corresponding to the adversary positions.  Following this, the state space is defined as:
    \begin{align*}
        \mathcal{S} \coloneqq [0,L]^n\times [0,L]^m.
    \end{align*} 
For the actions of robots, we have $a_i^t = (v_i^t, g_i^t)$, which is a hybrid combination of continuous speed $v_i^t \in [-v_{\max}, v_{\max}]$ and discrete guard behavior $g_i^t \in \{1,2, \dots, m\}$. The action space for all robots is
\begin{align}
    \mathcal{A} \coloneqq \left([-v_{\max}, v_{\max}] \times \{1,2, \dots, m\}\right)^n.
\end{align}

Based on $\mathcal{S}$ and $\mathcal{A}$, the state transition is $\mathcal{T}:\mathcal{S}\times\mathcal{A}\to\mathcal{S}$. The robot states follow motion dynamics \eqref{eq_dyn}, which is deterministic, and the adversary positions are static and do not depend on $\mathcal{A}$. 
The ${R}({\bm{s}}^t,{\bm{a}}^t) $ is the immediate reward of action $ \bm{a}^t \in\mathcal{A}$ with state $\bm{s}^t \in\mathcal{S}$, defined as the negative team cost
\begin{align}\label{eq_def_rw}
     {R}({\bm{s}}^t,{\bm{a}}^t) \coloneqq -\sum_{i=1}^{n}(R_i^t+P_i^t).
\end{align}

We complete our MDP formulation by choosing a discounter factor $\gamma=0.995$. The goal of RL is to learn a policy $\pi: \mathcal{S}\to\mathcal{A}$ to maximize the expected cumulative reward for the whole team over the task horizon $T$, i.e.,
\begin{align}
    \max_{\pi} \quad\underset{{\bm{a}}^t \sim \pi(\cdot|{\bm{s}}^t)}{\mathbb{E}}\left[\sum_{t=0}^T (\gamma)^t R^t\right]. 
    \label{eqn::reward}
\end{align}

\smallskip
\noindent\textbf{Multi-weighted hot state encoding.} To employ RL methods to solve our MDP problem, we can directly feed a scalar representation of each robot's and adversary's state (position) to the model. However, we observe that the dimension of our state space is much smaller than that of the action space due to the joint speed and guard behaviors. This discrepancy hinders the neural network's reasoning capabilities, which cannot efficiently learn parameters \cite{BM-RS-WJX-KNZ-BC-HD:19}. Inspired by the one-hot encoding, we seek to expand the state space using a similar mechanism. However, one-hot encoding is typically suited for discrete variables, whereas our state space is continuous. To address this, for robot positions, we introduce a new \textit{weighted hot encoding} mechanism, which represents a continuous variable as the weighted average of two neighboring one-hot vectors. Specifically, let $h(s)\in [0,1]^{L+1}$ denote the weighted hot encoding for a state $s\in[0,L]$, and $s=s_\text{int}+s_\text{dec}$, which has both integer and decimal parts.
Let $h(s)[k], k\in\mathbb{Z}$ denote the $k$th element of vector $h(s)$.
Then $h(s)$ is a vector with two nonzero entries: 
\begin{align}
    \begin{cases}
    h(s)[s_\text{int}+1]=1- s_\text{dec}\\ h(s)[s_\text{int}+2]=s_\text{dec}
    \end{cases} 
\end{align}
As a simple example, if $L=4$ and $s=3.2$. Since $3.2=0.8\times3+0.2\times 4$, one has, $h(s)=[0~~0~~0~~0.8~~0.2]^{\top}$. If $s=0.7=0\times0.3+1\times0.7$, one has $h(s)=[0.3~~0.7~~0~~0~~0]$.


Since we consider a centralized MDP, we vector stack each robot's encoding and adversaries encoding through $h(\cdot)$ to obtain a Multi-weighted hot state encoding as follows:
$$
    \widetilde{\bm{s}}^t=\text{vec}\{h(s_1^t),\cdots,h(s_n^t),h(z_1), \cdots, h(z_m)\},
$$
which has a dimension of $[0,1]^{(n+1)(L+1)}$. Ths allows us to individually encode robots' positions and adversary positions in a way that is easily understood by the neural network.
It is worth \textbf{remarking} that one-hot encoding is typically used for categorical variables. In our problem, the positions of robots, whether inside or outside of adversary-impacted zones, correspond to completely different properties and different feasible guard actions. This distinction favors one-hot encoding, as all input entries are orthogonal to each other. This fact further justifies why multi-weighted hot state encoding can enhance our learning efficiency.

\begin{figure}
    \centering
    \includegraphics[width=0.45\textwidth]{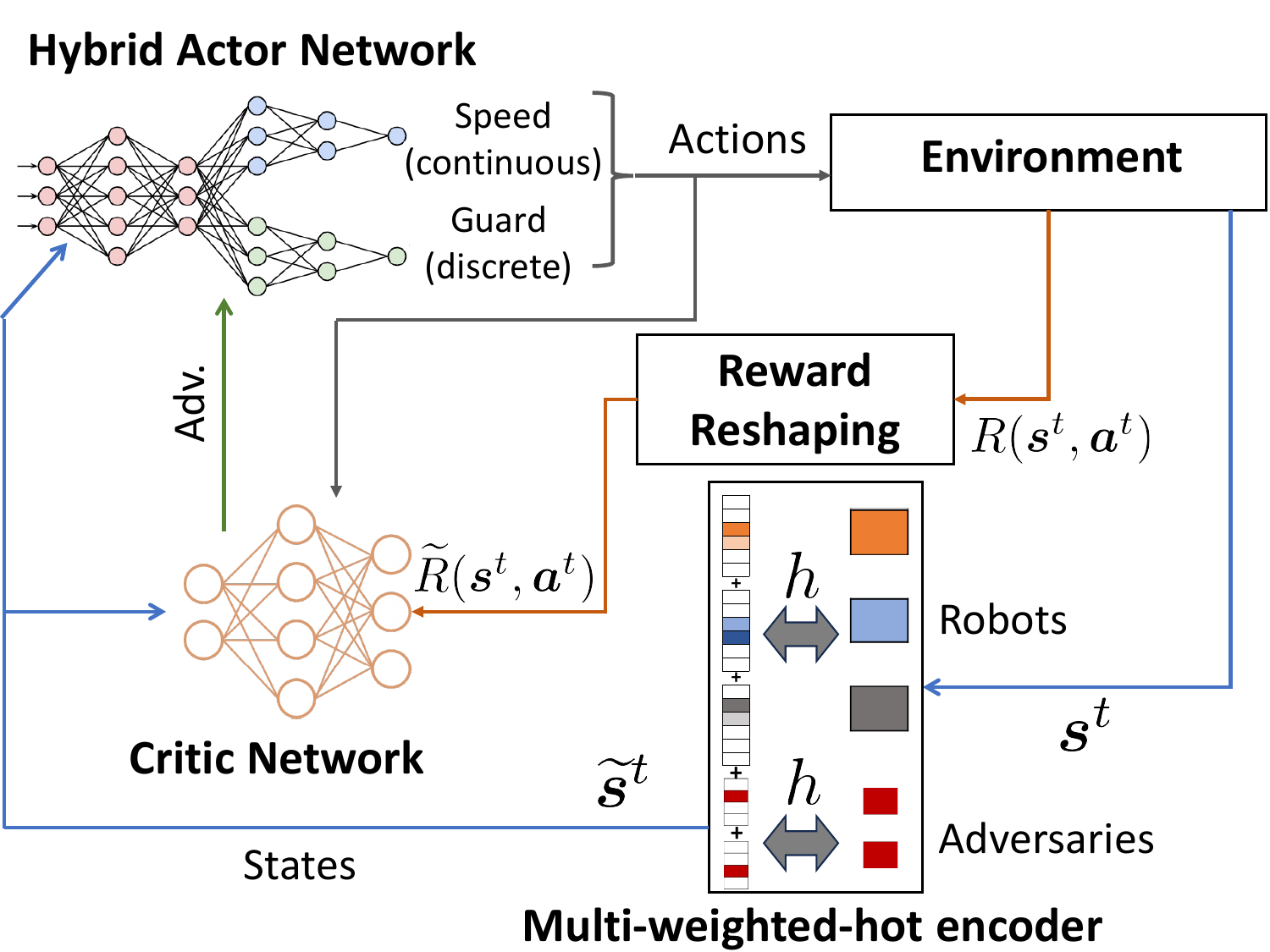}
    \caption{RL Implementation: H-PPO with multi-weighted hot encoding and reward reshaping}
    \vspace{-3ex}
    \label{Fig_HPPO}
\end{figure}

\smallskip
\noindent\textbf{Reward reshaping.} Since our task requires all robots to move to the terminal position, it is common to introduce a one-time constant reward $Q(\bm{s}^t)=q$, if ${s}_i^t=L, \forall i$; $Q(\bm{s}^t)=0$, otherwise. However, this terminal reward is so sparse that it provides limited guidance to robots for state-space exploration and policy updates. Due to the greedy nature of the action section process, robots are reluctant to enter adversary zones. 
To address this, we further introduce a reshaping reward 
\begin{align}
    F(\bm{s}^t, \bm{s}^{t+1})=c\sum_{i=1}^n\left(\gamma s_i^{t+1}-s_i^{t}\right),
\end{align}
which incites robots to move forward and speeds up the learning process\cite{DSM-KD:12}. The final reshaped reward reads 
\begin{align}
    \widetilde{R}(\bm{s}^t, \bm{a}^t) = {R}(\bm{s}^t, \bm{a}^t) + Q(\bm{s}^t) + F(\bm{s}^t, \bm{s}^{t+1})
\end{align}
where $\bm{s}^{t+1}$ is determined by $\bm{s}^t, \bm{a}^t$ through $\mathcal{T}$.
We \textbf{note} that the reshaped reward does not change the optimal solution compared to the original formulation. This is guaranteed by \cite{NAY-HD-RSJ:99}, as both $Q(\bm{s}^t)$ and $F(\bm{s}^t, \bm{s}^{t+1})$ can be rewritten into a potential-based function: $\gamma \Phi(\bm{s}^{t+1}) - \Phi(\bm{s}^t)$, where $\Phi(\cdot)$ is a real-valued function of state and $\gamma$ is the discount factor.

\noindent\textbf{RL Implementation.} Combining the formulated MPD with multi-weighted hot state encoding and reward reshaping, we use two proximal policy optimization (PPO) based RL algorithms to solve the multi-robot path traveling problem. The key difference lies in the way we handle the hybrid action space. First, for simplicity, we consider pure discrete action space, assuming robots only take integral speeds. This has led to a standard PPO with discrete actions (D-PPO) \cite{HCCY-MDC-HM:20}, and the proposed multi-weighted hot state encoding degrades to multi-one hot encoding. Second, we consider PPO with hybrid action space (H-PPO), and let the actor-network simultaneously output continuous speed actions and discrete guard actions. The policy losses ($log$ probabilities) of the two actions are combined and used for training the network parameters. A conceptual diagram of the RL implementation is shown in Fig. \ref{Fig_HPPO}, with a centralized structure to handle all robots' rewards and actions. Leaving as our future work, a possible decentralized implementation of the RL paradigm is to let each robot possess a local model of Fig.~\ref{Fig_HPPO}. Then, leveraging our multi-weighted-hot encoder, if the robot cannot observe certain robot's states, the corresponding weighted-hot vector has all entries being zero.

 \section{Simulated Experiments}


\begin{figure}[b]
    \centering
    \vspace{-3ex}
    \includegraphics[width=.4\textwidth]{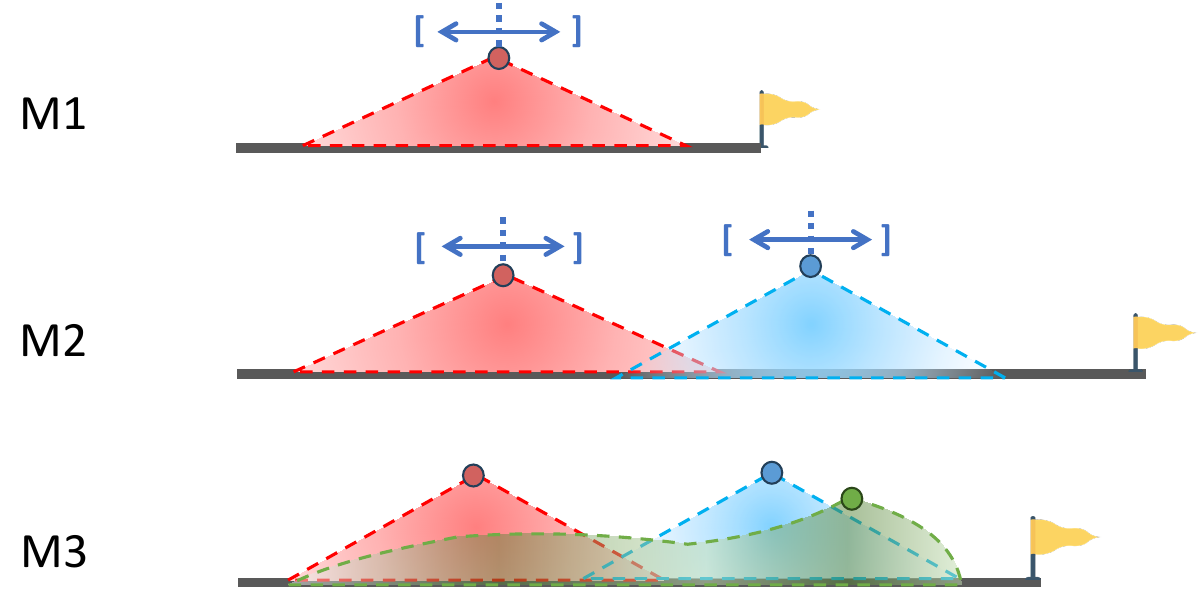}
    \caption{Experiment environments with different adversary configurations. The height represents the unit risk each adversary generates at different locations.}
    \label{fig_env}
\end{figure}
\begin{figure*}[t]
    \centering
    \includegraphics[width=.95\textwidth]{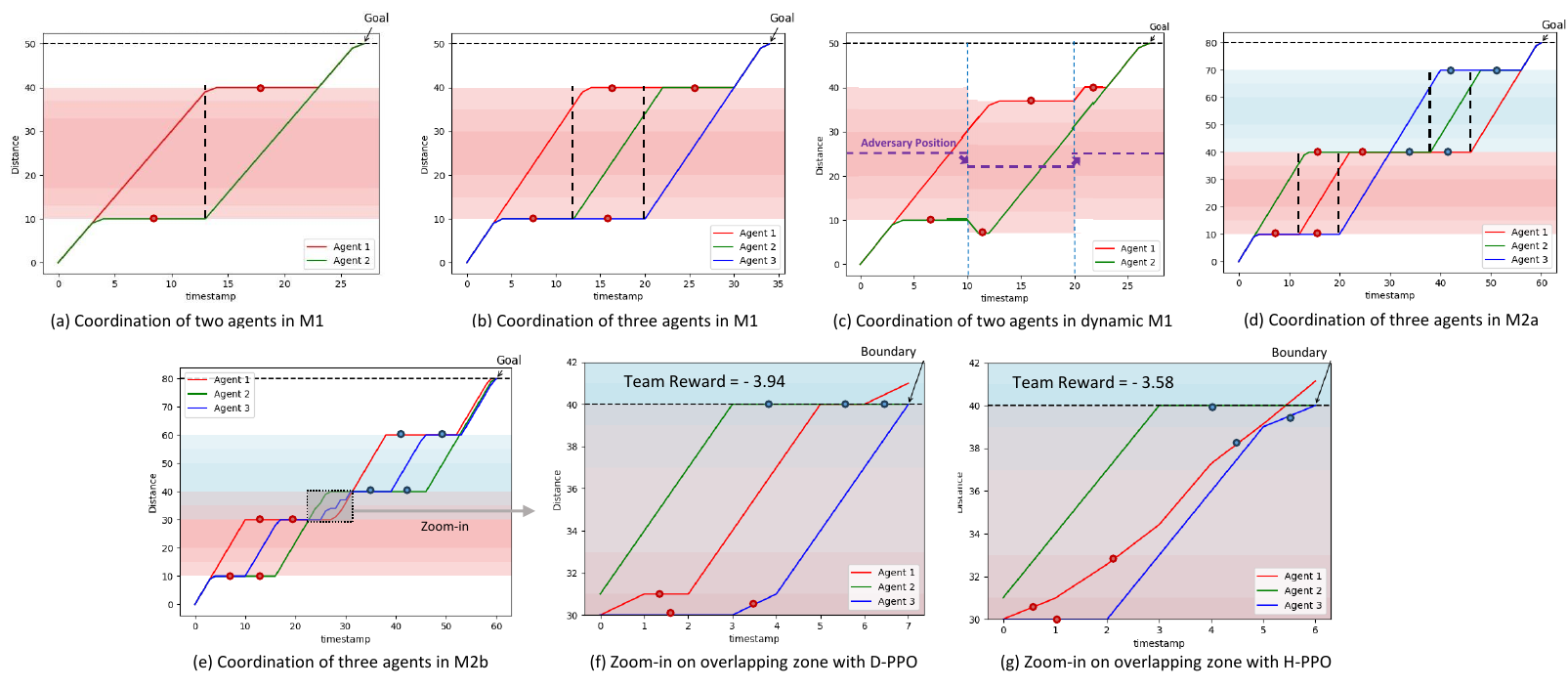}
    \caption{\small Using D-PPO and H-PPO to solve team coordination problem with first two environments in Fig. \ref{fig_env}. The $x$-axis represents time and the $y$-axis represents the length the robot traveled in the environment. The slope represents the robot's speed. The color dots on the trajectories represent the adversary the robot is currently guarding against. In (c) the purple dash represents the current adversary position.  
The shades are the adversary-impacted zones with darker colors in the middle to represent higher risk, corresponding to Fig.~\ref{fig_env}.  }
    \vspace{-3ex}
    \label{fig_M1M2a}
\end{figure*}
In this section, we present simulated experiment results to validate the analytical statements in Sec.~\ref{Sec_theory} and the RL implementation for team coordination behaviors in Sec.~\ref{sec_RL}. For complex cases, we discuss the reasoning behind these behaviors and why they reduce team costs. We also provide comparisons with baselines and numerical methods based on MINLP using Surrogate optimization~\cite{queipo2005surrogate} and BONMIN Solver~\cite{bonami2007bonmin}.

\subsection{Simulation environment}
We consider three different environments shown in Fig.~\ref{fig_env}, including 1-3 adversaries with potential overlaps over their impact zones. The route in Fig. \ref{Fig_intro} is abstracted as a linear distance from the starting point to the target.  
In M1, the adversary's position $z_j$ is randomly chosen from a discrete integer set $\mathcal{D}_j$, with width $10$ centered at the middle of the environment; In M2, the two adversary's positions are chosen from two discrete sets $\mathcal{D}_j$, each has a width $5$; the risk zones may overlap with each other. 
For both M1 and M2, when risk functions $f_j(s_i^t,z_j^t)$ are homogeneous and depends linearly on the distance between the robot and an adversary (visualized by the height of the shades in the figure). 
In M3, we consider stationary adversary positions, but the risk functions are heterogeneous and nonlinear, and with more complex overlaps.


We choose the following environment parameters: time interval $\Delta t=1$, max robot speed $v_{\max}=3$, risk co-efficient $\eta=1$, time penalty $p=1$, guard discount coefficient $\beta = 0.6$. 
All rewards, before being sent to D-PPO, H-PPO models, are re-scaled for normalization purposes.
For each environment (M1, M2, M3), the learning model is individually trained. Within each environment, since adversary positions are part of the state space, a single model is trained with all possible adversary positions. During execution, this single model can adapt to adversary positions that are randomly generated for each test. This reflects the generalization capability of our model.



\subsection{Team coordination and model generalization with homogeneous adversaries.}
The results in Fig.~\ref{fig_M1M2a}a-c consider the M1 scenario with two and three robots, respectively. The D-PPO and H-PPO methods generate almost the same results except for slight differences in velocity when robots leave or approach the boundaries of adversary-impacted zones, which leads to minor changes to the final reward. For conciseness, we only visualize the results from D-PPO.
In Fig.~\ref{fig_M1M2a}a, the behaviors of the two robots follow a `bounding overwatch' pattern described in Sec. \ref{Sec_theory}, i.e., one robot guards at the boundary of the adversary-impacted zone with minimal risk until the other robot moves across the zone at full speed. Then, the two robots switch roles to cooperatively accomplish the task with minimum cost. 
In Fig.~\ref{fig_M1M2a}b, as the number of robots increases, we observe (from the vertical dashes) a change in robots' coordination such that the guarding robots start moving $2$ seconds before the traveling robots arrive at the other end. This adjustment is due to the increase in the number of robots; the time penalty encourages the robots to arrive at the destination more quickly.
In Fig.~\ref{fig_M1M2a}c, we change the adversary's position in the middle of the execution. Although our algorithm is not designed to account for the dynamic behaviors of adversaries, the closed-loop nature of the policy and the fact that all possible adversary positions have been learned during the training process allow robots to quickly make adaptations. At the behavior level, robots maintain their position at the boundary of the zone when guarding. This demonstrates the model's generalization capability. If using optimization-based approaches, such as mixed-integer programming, then replanning is necessary.

We employ D-PPO and H-PPO methods to solve optimal coordination under the M2 environment.
Since the positions of the two adversaries are randomly initialized, it is possible that the two adversary-impacted zones may or may not overlap with each other, which will then impact the coordination patterns of the robots. Here, we choose two representative cases: without overlap and with overlap. In Fig.\ref{fig_M1M2a}d, without overlap, the three robots simply reproduce coordination behaviors in Fig. 4b over the two zones, respectively. 
In the case of Fig.\ref{fig_M1M2a}e, where two adversaries have an overlapped area (c.f. M2 in Fig. \ref{fig_env}), the results of D-PPO and H-PPO show relatively consistent strategies for the $[10,30]$ and $[40,60]$ zones, as shown in Fig.~\ref{fig_M1M2a}e, but exhibit variant behaviors in the overlapping zone $[30,40]$. 
To investigate this, we zoom into the overlapping zone and observe the difference between H-PPO and D-PPO results shown in Fig.~\ref{fig_M1M2a}f-g. Note that robots start at $\{30,31,30\}$ instead of $\{30,30,30\}$ because when entering the zone, robot 2 moves with $v_{\max}=3$ from $s_2=28$ and directly arrives at $s=31$. For a similar reason, the terminal states are $\{41,40,40\}$.
In both Fig.~\ref{fig_M1M2a}f and g, at least one robot takes intermediate speeds that perform move and guard behaviors simultaneously. 
This aligns with the hypothesis-(ii) at the end of Sec.~\ref{Sec_theory}. 
When robots are close to the boundaries of the overlapping zone, their risks are dominated by red and blue adversaries, respectively. 
As observed in plots, the stationary robot always guards the adversary which causes more risk, while the robot with intermediate speed will guard the other adversary.
Moreover, a comparison of the results reveals that H-PPO, with its ability to handle continuous speeds, can refine strategies more effectively, resulting in better rewards. The asymmetry in robot's behavior may be due to time discretization, where the cost is computed at the end of each time step. 

\subsection{Validation of Reward Reshaping and Weighted-hot State Encoding}
Using the same environment setup as in M1 and M2, we validate the effectiveness of the proposed reward reshaping and weighted-hot state encoding techniques, and compare them with the results when these techniques are not used.

\begin{figure}[h]
    \centering
    \includegraphics[width=.46\textwidth]{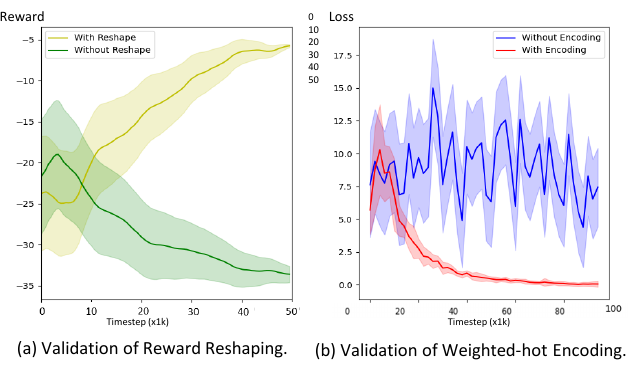}
    \caption{Validations of Reward Reshaping and Weighted-hot State Encoding Techniques.}
    \label{fig:Validations}
\end{figure}

\subsubsection{Necessity of Reward Reshaping}
Using environment M1, the effectiveness of reward reshaping is demonstrated in Fig.~\ref{fig:Validations} (a). It is observed that without reward reshaping, the reward curve in the PPO training shows a trend of downward convergence. When executing the learned policy in the gym environment, it was noticed that all robots halted before entering the adversary zone. We believe this is due to the original problem formulation, where entering the adversary zone triggers significant costs immediately, causing unfavorable exploration because robots tend to avoid such heavily penalized actions. 
Furthermore, since the environment has a maximum simulation times step of 100\footnote{We note that in the ideal case, 100 steps are more than enough for all robots to reach the destination.}, robots will simply wait until the simulation stops and learn a bad behavior with poor rewards. 
On the other hand, with reward reshaping applied, the training could stably converge toward the optimal solution's reward values. As we execute the learned policy in the gym environment, robots can generate meaningful coordination behaviors. This verifies the effectiveness of our reward reshaping as discussed in Sec. \ref{sec_RL}.


\subsubsection{Effectiveness of Weighted-hot State Encoding.}
We use the M2 environment to verify the effectiveness of our weighted-hot state encoding mechanism. 
From Fig.~\ref{fig:Validations} (b), the training loss curve indicates that directly using the scalar values of robots' and adversaries' positions as neural network inputs, i.e., without weighted-hot encoding results in continuous oscillations in the PPO training process. This implies that the neural network faces difficulties in understanding the state description. Instead, when weighted-hot encoding is used and all other model parameters (latent layers, learning rate, exploration rate) are kept the same, the training stability is significantly improved and the loss converges to zero quickly. This verifies the effectiveness and necessity of our weighted-hot state encoding mechanism.




\subsection{Behavior analysis of team coordination in complex environments with heterogeneous adversaries.}
\begin{figure}[t]
    \centering
    \includegraphics[width=.42\textwidth]{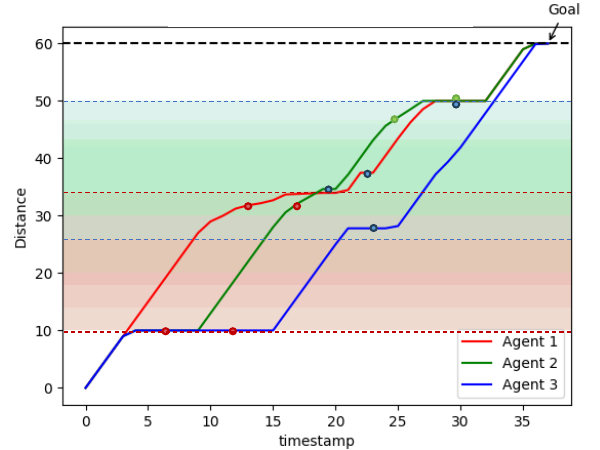}
    \caption{Coordination of three robots in M3 using H-PPO. The figure is read the same way as Fig.~\ref{fig_M1M2a}. }
    \vspace{-1ex}
    \label{fig_n3m3}
\end{figure}
For the case of three adversaries with multiple overlaps and heterogeneous nonlinear risk functions (M3), we deploy three robots and obtain coordination results as shown in Fig.~\ref{fig_n3m3}. Here, the adversaries are stationary. We added colored dashes to better visualize the boundaries of red and blue adversaries. The environment's complexity makes it difficult to judge the optimality of the obtained coordination. In the following, we only discuss the reasoning behind the obtained result and explain why it reduces the overall team cost. First, according to Fig. \ref{fig_env}, the M3 environment adds M2 with an extra green adversary. This green adversary generates an unsymmetric and nonlinear risk zone, which poses little risk early in the path but grows large as robots proceed. Consequently, in Fig.~\ref{fig_n3m3}, robots first follow a pattern very similar to that of Fig. \ref{fig_M1M2a}e. However, midway through the path, as risks associated with the green adversary become large, the predecessor robot 1 does not fully stop to guard others. Instead, it takes an intermediate speed to guard the red adversary. This lasts until robot 1 meets robot 2 and both leave the boundary of the red adversary. On the other hand, we observe robot 3 stops in the middle of the path to perform guard. This happens because, at the moment, the other two robots are suffering huge risks from both blue and green adversaries. By stopping and thereby increasing its own risk, robot 3 contributes to the cost-saving of the whole team. Finally, when robot 1 and 2 both arrives at the boundary, they help robot 3 by guarding red and blue adversaries, respectively. We note that the coordination presented in Fig.~\ref{fig_n3m3} may not be the global optimal, and its optimality gap is difficult to quantify. However, as discussed above, the observed robot coordination does exhibit rational behaviors, with the goal of reducing the overall team cost.

\begin{figure}[t]
    \centering
    \vspace{-1ex}
    \includegraphics[width=.47\textwidth]{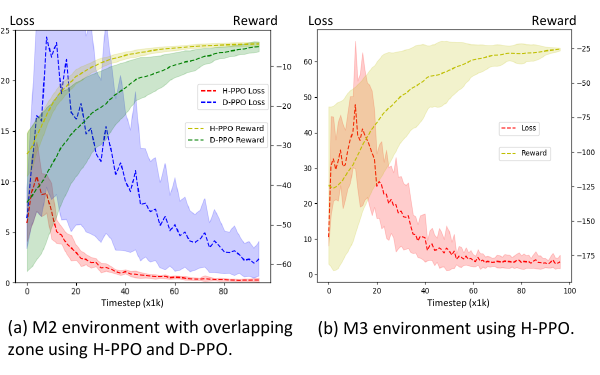}
    \caption{Convergence results for training loss and rewards.}
    \vspace{-3ex}
    \label{fig_loss}
\end{figure}

\setlength{\tabcolsep}{2pt}
\begin{table}[b]
    \centering
    \vspace{-1ex}
    \caption{Comparison with Baseline and MINLP Solvers}
    \begin{tabular}{cccccc}
         \toprule
         Env & D-PPO& \textbf{H-PPO} & Baseline & ~BONMIN~ & Surrogate\\ \midrule 
        M1  & $-5.80\pm 0.2$ & $\mathbf{-5.78\pm 0.0}$ & $-9.94$&$\mathbf{-5.78}$&$-6.11$\\ \midrule
        M2 & $-8.63 \pm 0.7$ & $-8.22 \pm 0.4$ & $-20.18$ &$\mathbf{-8.03}$& N/A  \\ \midrule
        M3  & $-40.13 \pm 8.5$ & $\mathbf{-26.87\pm 2.0}$ & $-44.70$ & $-30.32$&N/A \\ \bottomrule
    \end{tabular}
    \label{rewards_comparison}
    \vspace{-1ex}
\end{table}
\subsection{Reward Comparison with Baseline and MINLP Solvers}
To visualize the training process, in Fig.~\ref{fig_loss}, we use M2 and M3 environments as representative results to show training losses and rewards. We observe that H-PPO performs better than D-PPO and we note that D-PPO struggles to converge for the M3 environment.

For comparison purposes, we introduce a naive baseline strategy where all robots' actions are decoupled, and each robot individually uses greedy choices for move and guard. Additionally, we write the MINLP formulation of the problem and solve it using the BONMIN Solver~\cite{bonami2007bonmin} in Python and the Surrogate algorithm in MATLAB. We note that the piece-wise and condition-dependent non-linear functions \eqref{eq_def_pi}-\eqref{eq_discount} require large numbers of auxiliary variables to formulate, which leads to low computational efficiency. When run indefinitely, BONMIN takes hours trying to close the optimality gap. Therefore, we stop BONMIN in 10 minutes if it does not converge and record the result.
For all environments in Fig. \ref{fig_env}, a fixed set of adversary positions is chosen, and three robots are deployed to test all approaches. Specifically for M2, there is an overlap between the two adversaries. As shown in Table \ref{rewards_comparison}, the developed H-PPO generally outperforms or provides comparable performance to other methods. The Surrogate does not exhibit meaningful convergence on rewards for the two complex cases. As a final remark, while BONMIN is stopped after 10 minutes, H-PPO obviously requires significantly more training time. However, H-PPO can be trained offline and provides a general closed-loop policy for all possible adversary positions. In contrast, for BONMIN, each adversary position needs to be solved individually for an open-loop solution. Thus, for real-time applications, the proposed H-PPO is more desirable.





 \section{Conclusion and Future Work}
We have formulated a coordination problem considering a team of robots traversing a route with adversaries. The cost of the team was determined by the time penalty and the risk robots accumulated when crossing adversary-impacted zones, which can be reduced by robots' guard behavior when moving at a lower speed. We analyzed the optimal coordination strategy for a single adversary scenario. For complex environments, we proposed and implemented an H-PPO method with reward reshaping and a new multi-weighted hot state encoding mechanism. Simulated experiments are performed under different environments and compared with alternative approaches to validate our analysis and the effectiveness of the H-PPO method. 
Based on the simulation results, we discussed the reasoning behind these behaviors in terms of reducing the overall team cost. Future work will consider developing a decentralized learning paradigm to achieve scalable coordination with a larger number of robots and more complicated environments. We also seek to expand the proposed formulation to non-route-based environments considering the geometries of risk zones and terrain, and the decentralization of the proposed algorithm.   

\bibliographystyle{ieeetr}
\bibliography{bib}
\end{document}